\documentclass{article}
\usepackage{style/unites}
\usepackage{XCharter}
\usepackage[scaled=1.1]{zlmtt}

\usepackage[utf8]{inputenc}
\usepackage[T1]{fontenc}
\usepackage{microtype}

\usepackage{amsmath}
\usepackage{amssymb}
\usepackage{amsfonts}
\usepackage{amsthm}
\usepackage{mathtools}
\usepackage{mathrsfs}
\usepackage{physics}
\usepackage{braket}
\usepackage{slashed}
\usepackage{nicefrac}
\usepackage{textcomp}
\usepackage{dsfont}
\usepackage{bbm}
\usepackage{bm}

\usepackage{graphicx}
\usepackage{subcaption}
\usepackage[export]{adjustbox}
\usepackage{float}
\usepackage{booktabs}
\usepackage{dcolumn}
\newcolumntype{d}[1]{D{.}{.}{#1}}
\usepackage{bigstrut, tabularx, multirow, makecell, diagbox}
\usepackage{colortbl}
\usepackage{tabularray}
\UseTblrLibrary{booktabs}
\usepackage{threeparttable}
\usepackage{tablefootnote}
\usepackage{fontawesome5}

\usepackage{placeins}
\usepackage{caption}
\usepackage{footnote}
\usepackage{enumitem}
\usepackage{multicol}
\usepackage{xspace}
\usepackage{titletoc}
\usepackage{titlesec}
\usepackage[bottom]{footmisc}
\usepackage{setspace}

\usepackage{wrapfig}
\usepackage{tikz}
\usepackage{quantikz}
\usepackage{dashbox}
\usepackage{mdframed}
\usepackage{marvosym}
\usepackage{pifont}
\usepackage{CJK}
\usepackage{url}

\usepackage[table,x11names]{xcolor}
\usepackage[most]{tcolorbox}
\tcbuselibrary{breakable}
\usetikzlibrary{decorations.pathreplacing, fit}

\definecolor{primaryblue}{HTML}{0066CC}
\definecolor{accentcyan}{HTML}{00D4AA}
\definecolor{warmorange}{HTML}{FF6B35}
\definecolor{deepgray}{HTML}{2C3E50}
\definecolor{lightgray}{HTML}{F8F9FA}
\definecolor{gradientstart}{HTML}{667eea}
\definecolor{gradientend}{HTML}{764ba2}

\definecolor{citecolor}{HTML}{0071bc}
\definecolor{citeblue}{RGB}{0, 113, 188}
\definecolor{linkcolor}{HTML}{9A4D92}
\definecolor{firebrick}{rgb}{0.698,0.133,0.133}

\definecolor{paleviolet}{HTML}{E1EEFC}
\definecolor{CarolinaUltraLight}{HTML}{E7F4FC}
\definecolor{lightgrey}{RGB}{247, 247, 247}
\definecolor{shadecolor}{HTML}{EFEFEF}
\definecolor{lightyellow}{rgb}{1.0, 0.95, 0.7}
\definecolor{lightblue}{rgb}{0.90, 0.95, 1.0}
\definecolor{light-gray}{gray}{0.95}

\definecolor{darkgrey}{rgb}{0.5, 0.5, 0.5}
\definecolor{darkgreen}{rgb}{0, 0.5, 0}
\definecolor{mydarkblue}{rgb}{0,0.08,0.45}
\definecolor{mydarkblue2}{rgb}{0.133, 0.133, 0.698}
\definecolor{echodrk}{HTML}{0099cc}
\definecolor{mymauve}{rgb}{0.58,0,0.82}
\definecolor{midnightblue}{rgb}{0.1,0.1,0.44}
\definecolor{oxfordblue}{rgb}{0.0,0.13,0.28}
\definecolor{prussianblue}{rgb}{0.0,0.19,0.33}
\definecolor{coolteal}{rgb}{0, 0.45, 0.45}
\definecolor{olive}{rgb}{0.1, 0.3, 0}
\definecolor{mypurple}{rgb}{0.5,0,0.5}
\definecolor{almond}{rgb}{0.94, 0.87, 0.8}

\definecolor{blue_ampEncoding}{HTML}{DAE8FC}
\definecolor{green_encoder}{HTML}{D5E8D4}
\definecolor{purple_decoder}{HTML}{E1D5E7}
\definecolor{yellow_measure}{HTML}{FFF2CC}
\definecolor{gray_block}{HTML}{F5F5F5}
\definecolor{pink_dru}{HTML}{FAD9D5}
\definecolor{orange_v}{HTML}{FAD7AC}

\definecolor{colorA}{rgb}{1,0,0}
\definecolor{colorB}{rgb}{0,0.3,1}
\definecolor{colorC}{rgb}{0.9,0.8,0.2}
\definecolor{colorD}{rgb}{0,0.65,0}
\definecolor{lesslightgray}{rgb}{0.5,0.5,0.5}
\definecolor{fundamental}{RGB}{55, 110, 111}
\definecolor{Gred}{RGB}{219, 50, 54}
\definecolor{ToCgreen}{RGB}{0, 128, 0}
\definecolor{Sepia}{RGB}{112, 66, 20}
\definecolor{Dblue}{rgb}{0,0.08,0.45}
\definecolor{Blue}{rgb}{0, 0, 0.8}
\definecolor{blue}{rgb}{0,0,1}
\definecolor{UNCblue!10}{rgb}{0.84,0.91,0.98}
\definecolor{RowAlt}{rgb}{0.98,0.98,0.99}

\definecolor{CarolinaBlue}{HTML}{7BAFD4}        
\definecolor{CarolinaLightBlue}{HTML}{B3D4E5}   
\definecolor{CarolinaUltraLight}{HTML}{E8F4F8}  
\definecolor{CarolinaText}{HTML}{1C2B33}        

\usepackage[pagebackref=true,breaklinks=true,colorlinks,hyperfootnotes=false]{hyperref}
\hypersetup{
  colorlinks,
  citecolor=citeblue,
  linkcolor=firebrick,
  urlcolor=firebrick
}
\usepackage[nameinlink,capitalize,noabbrev]{cleveref}

\titlespacing\section{0pt}{4pt plus 4pt minus 2pt}{-2pt plus 2pt minus 2pt}
\titlespacing\subsection{0pt}{2pt plus 4pt minus 2pt}{-2pt plus 2pt minus 2pt}
\titlespacing\subsubsection{0pt}{2pt plus 4pt minus 2pt}{-2pt plus 2pt minus 2pt}

\makeatletter
\def\th@remark{%
  \thm@headfont{\bfseries}%
  \normalfont 
  \thm@preskip\topsep \divide\thm@preskip\tw@
  \thm@postskip\thm@preskip
}
\makeatother

\theoremstyle{definition}

\newtheorem{theorem}{Theorem}[section]
\tcolorboxenvironment{theorem}{
  breakable,
  colback=black!10,
  colframe=white,
  width=\linewidth, 
  enlarge left by=0pt,
  enlarge right by=0pt,
  boxsep=5pt,
  boxrule=0pt,
  left=0pt,right=0pt,top=0pt,bottom=0pt,
  arc=8pt,
  before skip=\topsep,
  after skip=\topsep
}

\newtheorem{lemma}{Lemma}[section]
\tcolorboxenvironment{lemma}{
  breakable,
  colback=black!10,
  colframe=white,
  width=\linewidth,
  enlarge left by=0pt,
  enlarge right by=0pt,
  boxsep=5pt,
  boxrule=0pt,
  left=0pt,right=0pt,top=0pt,bottom=0pt,
  arc=8pt,
  before skip=\topsep,
  after skip=\topsep
}

\tcolorboxenvironment{corollary}{
  breakable,
  colback=black!10,
  colframe=white,
  width=\linewidth,
  enlarge left by=0pt,
  enlarge right by=0pt,
  boxsep=5pt,
  boxrule=0pt,
  left=0pt,right=0pt,top=0pt,bottom=0pt,
  arc=8pt,
  before skip=\topsep,
  after skip=\topsep
}

\tcolorboxenvironment{proposition}{
  breakable,
  colback=black!10,
  colframe=white,
  width=\linewidth,
  enlarge left by=0pt,
  enlarge right by=0pt,
  boxsep=5pt,
  boxrule=0pt,
  left=0pt,right=0pt,top=0pt,bottom=0pt,
  arc=8pt,
  before skip=\topsep,
  after skip=\topsep
}

\tcolorboxenvironment{definition}{
  breakable,
  colback=black!10,
  colframe=white,
  width=\linewidth,
  enlarge left by=0pt,
  enlarge right by=0pt,
  boxsep=5pt,
  boxrule=0pt,
  left=0pt,right=0pt,top=0pt,bottom=0pt,
  arc=8pt,
  before skip=\topsep,
  after skip=\topsep
}

\newtheorem{assumption}{Assumption}[section]
\tcolorboxenvironment{assumption}{
  breakable,
  colback=black!10,
  colframe=white,
  width=\linewidth,
  enlarge left by=0pt,
  enlarge right by=0pt,
  boxsep=5pt,
  boxrule=0pt,
  left=0pt,right=0pt,top=0pt,bottom=0pt,
  arc=8pt,
  before skip=\topsep,
  after skip=\topsep
}

\tcolorboxenvironment{claim}{
  breakable,
  colback=black!10,
  colframe=white,
  width=\linewidth,
  enlarge left by=0pt,
  enlarge right by=0pt,
  boxsep=5pt,
  boxrule=0pt,
  left=0pt,right=0pt,top=0pt,bottom=0pt,
  arc=8pt,
  before skip=\topsep,
  after skip=\topsep
}

\tcolorboxenvironment{problem}{
  breakable,
  colback=black!10,
  colframe=white,
  width=\linewidth,
  enlarge left by=0pt,
  enlarge right by=0pt,
  boxsep=5pt,
  boxrule=0pt,
  left=0pt,right=0pt,top=0pt,bottom=0pt,
  arc=8pt,
  before skip=\topsep,
  after skip=\topsep
}

\tcolorboxenvironment{question}{
  breakable,
  colback=black!10,
  colframe=white,
  width=\linewidth,
  enlarge left by=0pt,
  enlarge right by=0pt,
  boxsep=5pt,
  boxrule=0pt,
  left=0pt,right=0pt,top=0pt,bottom=0pt,
  arc=8pt,
  before skip=\topsep,
  after skip=\topsep
}

\newtcolorbox{titleblock}{
  enhanced,
  frame hidden,
  colback=CarolinaUltraLight,
  colframe=CarolinaUltraLight,
  boxrule=0pt,
  arc=10pt,
  left=14pt,
  right=14pt,
  top=14pt,
  bottom=14pt,
  width=\linewidth,
  before skip=12pt plus 4pt,
  after skip=12pt plus 4pt,
  grow to left by=1.5pt,
  grow to right by=1.5pt,
  before upper={
    \setlength{\parindent}{0cm}
    \setlength{\parskip}{0.5cm}
  }
}

\crefname{theorem}{Theorem}{Theorems}
\crefname{proposition}{Proposition}{Propositions}
\crefname{lemma}{Lemma}{Lemmas}
\crefname{corollary}{Corollary}{Corollaries}
\crefname{definition}{Definition}{Definitions}
\crefname{assumption}{Assumption}{Assumptions}
\crefname{remark}{Remark}{Remarks}
\crefname{problem}{Problem}{Problems}
\crefname{property}{Property}{property}
\crefname{question}{Question}{Questions}

\numberwithin{equation}{section}
\numberwithin{theorem}{section}
\numberwithin{proposition}{section}
\numberwithin{definition}{section}
\numberwithin{lemma}{section}
\numberwithin{assumption}{section}
\numberwithin{remark}{section}

\def\1{\bm{1}}

\makeatletter
\let\save@mathaccent\mathaccent
\newcommand*\if@single[3]{%
    \setbox0\hbox{${\mathaccent"0362{#1}}^H$}%
    \setbox2\hbox{${\mathaccent"0362{\kern0pt#1}}^H$}%
    \ifdim\ht0=\ht2 #3\else #2\fi
}
\newcommand*\rel@kern[1]{\kern#1\dimexpr\macc@kerna}
\newcommand*\widebar[1]{\@ifnextchar^{{\wide@bar{#1}{0}}}{\wide@bar{#1}{1}}}
\newcommand*\wide@bar[2]{\if@single{#1}{\wide@bar@{#1}{#2}{1}}{\wide@bar@{#1}{#2}{2}}}
\newcommand*\wide@bar@[3]{%
    \begingroup
    \def\mathaccent##1##2{%
        \let\mathaccent\save@mathaccent
        \if#32 \let\macc@nucleus\first@char \fi
        \setbox\z@\hbox{$\macc@style{\macc@nucleus}_{}$}%
        \setbox\tw@\hbox{$\macc@style{\macc@nucleus}{}_{}$}%
        \dimen@\wd\tw@
        \advance\dimen@-\wd\z@
        \divide\dimen@ 3
        \@tempdima\wd\tw@
        \advance\@tempdima-\scriptspace
        \divide\@tempdima 10
        \advance\dimen@-\@tempdima
        \ifdim\dimen@>\z@ \dimen@0pt\fi
        \rel@kern{0.6}\kern-\dimen@
        \if#31
        \overline{\rel@kern{-0.6}\kern\dimen@\macc@nucleus\rel@kern{0.4}\kern\dimen@}%
        \advance\dimen@0.4\dimexpr\macc@kerna
        \let\final@kern#2%
        \ifdim\dimen@<\z@ \let\final@kern1\fi
        \if\final@kern1 \kern-\dimen@\fi
        \else
        \overline{\rel@kern{-0.6}\kern\dimen@#1}%
        \fi
    }%
    \macc@depth\@ne
    \let\math@bgroup\@empty \let\math@egroup\macc@set@skewchar
    \mathsurround\z@ \frozen@everymath{\mathgroup\macc@group\relax}%
    \macc@set@skewchar\relax
    \let\mathaccentV\macc@nested@a
    \if#31
    \macc@nested@a\relax111{#1}%
    \else
    \def\gobble@till@marker##1\endmarker{}%
    \futurelet\first@char\gobble@till@marker#1\endmarker
    \ifcat\noexpand\first@char A\else
    \def\first@char{}%
    \fi
    \macc@nested@a\relax111{\first@char}%
    \fi
    \endgroup
    }
\makeatother

\DeclareMathAlphabet{\mathsfit}{\encodingdefault}{\sfdefault}{m}{sl}
\SetMathAlphabet{\mathsfit}{bold}{\encodingdefault}{\sfdefault}{bx}{n}

\let\tilde\widetilde
\let\hat\widehat

\renewcommand{\arraystretch}{1.15}
\usepackage{algpseudocode}

\newcommand{\ourmethod}{CAST\xspace}

\makeatletter
\let\c@lemma\c@theorem
\let\c@proposition\c@theorem
\let\c@definition\c@theorem
\let\c@assumption\c@theorem
\let\c@remark\c@theorem
\let\c@claim\c@theorem
\makeatother

\begin{document}

\makeatletter
\def\blfootnote{\gdef\@thefnmark{}\@footnotetext}
\makeatother

\makeatletter
\pagestyle{fancy}
\fancyhf{}
\renewcommand{\headrulewidth}{1pt}
\chead{\small\bf CAST: Context- and Anomaly Structure-Conditioned Time Series Anomaly Generation
}
\cfoot{\thepage}
\thispagestyle{fancy}
\makeatother

\makeatletter
\def\icmldate#1{\gdef\@icmldate{#1}}
\icmldate{\today}
\makeatother

\makeatletter
\fancypagestyle{fancytitlepage}{
  \fancyhead{}
  \lhead{\includegraphics[height=0.8cm]{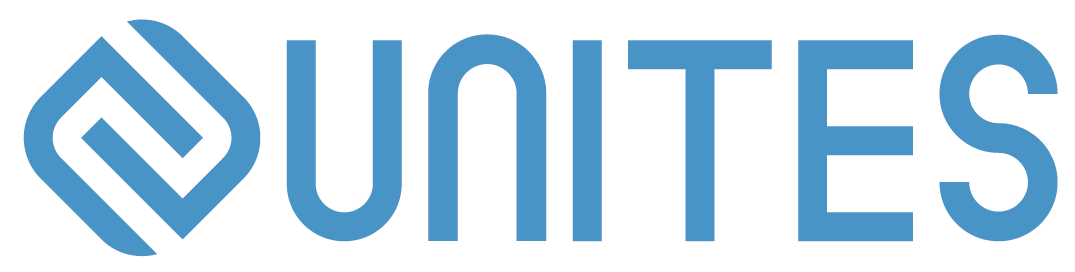}}
  \rhead{\it \@icmldate}
  \cfoot{}
}
\makeatother

\thispagestyle{fancytitlepage}

\vspace*{0.5em}

\noindent
\begin{titleblock}
    {\setlength{\parskip}{0cm}
     \raggedright
     {\setstretch{1.2}
      \LARGE\bfseries
      CAST: Context- and Anomaly Structure-Conditioned Time Series Anomaly Generation
      \par}
    }
    \vskip 0.2cm

    \begin{icmlauthorlist}
\mbox{Haochen Zhang$^{\,1}$},
\mbox{Jie Peng$^{\,1\,}$},
\mbox{Songyuan Sui$^{\,2\,}$},
\mbox{Yu-Chao Huang$^{\,1\,}$},
\mbox{Xiangqi Zhu$^{\,3\,}$},
and \mbox{Tianlong Chen$^{\,1\,\textrm{\Letter}}$}
\end{icmlauthorlist}

$^{1\,}$UNITES Lab, University of North Carolina at Chapel Hill
\quad $^{2\,}$Rice University
\quad $^{3\,}$Oregon State University

\{haochenz, morris, tianlong\}@cs.unc.edu,
 ss275@rice.edu, xiangqi.zhu@oregonstate.edu

$^{\textrm{\Letter}}$ Corresponding Author

    \vskip 0.2cm

    Anomalous time series play a critical role in safety-critical domains, yet they are inherently scarce, heterogeneous, and costly to obtain. Existing time series generation methods predominantly focus on synthesizing normal data, providing limited value when anomalous samples are needed. We identify two fundamental challenges in anomaly generation: (i) the scarcity of anomaly data, and (ii) the heterogeneous morphological characteristics of anomalies.
To address these challenges , we propose \textbf{CAST}, a \underline{\textbf{C}}ontext- and \underline{\textbf{A}}nomaly \underline{\textbf{S}}tructure-conditioned \underline{\textbf{T}}ime series anomaly generation framework with principled two-stage pretraining and finetuning strategy. In pretraining stage, we leverage abundant normal time series data to learn underlying system dynamics and substantially mitigate the limited availability of anomaly data. During finetuning, CAST explicitly conditions the generator on learned anomaly structure representations, enabling it to capture heterogeneous anomaly morphologies under similar contextual conditions.
Extensive experiments on multiple real-world univariate and multivariate datasets demonstrate that CAST consistently outperforms state-of-the-art anomaly generation methods in terms of both generation fidelity and downstream task utility, highlighting the effectiveness of the proposed approach.

    \vskip 0.2cm
    {\setlength{\parskip}{0cm}
    }
\end{titleblock}

\blfootnote{%
$^{\textrm{\Letter}}$ Corresponding authors: \{tianlong\}@cs.unc.edu
\\[2.5em]
\ifcsname @icmlpreprint\endcsname
  \textit{\csname @icmlpreprint\endcsname}%
\fi
}

\section{Introduction}

Time series generation \cite{naiman2024generative,yuan2024diffusionts,hu2024flowts,chang2023llm4ts,das2024decoder} aims to model and synthesize the temporal evolution of complex dynamical systems, and has recently emerged as a powerful paradigm for data-driven simulation, augmentation, and representation learning. Time series generative models can capture underlying system dynamics and synthesize realistic temporal samples to support a wide range of downstream applications. However, existing methods predominantly focus on synthesizing normal data, which is already abundant in real-world scenarios due to mature sensing infrastructures and long-term data collection. As a result, synthesized normal time series data offers limited marginal benefit \cite{chandola2009anomaly}. In contrast, anomalous time series are typically rare and costly to observe but still play a critical role in safety-critical applications, e.g., fault diagnosis \cite{jin2022time}, and risk assessment \cite{chen2025explainable}. Consequently, time series anomaly generation is impactful, motivating research going beyond normal time series generation \cite{darban2025genias,singh2022c}.

The challenge of anomalous time series generation is fundamentally rooted in the scarcity of anomalous data \cite{zamanzadeh2024deep}. In real-world systems, anomalies are rare by nature, which leads to severe overfitting risks when learning generative models directly from limited anomaly observations \cite{zamanzadeh2024deep}. \citet{singh2022c} explored pretraining--finetuning strategies to leverage abundant normal time series data under a VAE framework. Such approaches demonstrate that exploiting normal data is effective for mitigating anomaly data scarcity and stabilizing training.

\begin{wrapfigure}{r}{0.5\linewidth}
    \centering
    \includegraphics[width=\linewidth]{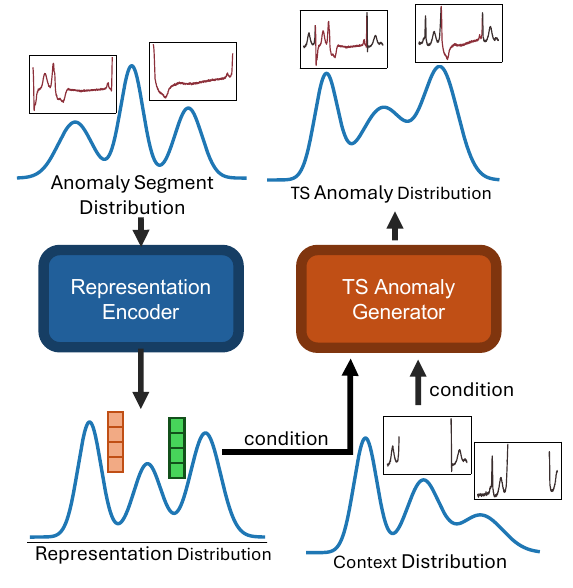}
    \caption{Overview of CAST. The representation encoder first infers a anomaly structure latent from a known anomaly segment. The time series anomaly generator takes the anomaly structure latent and normal context as conditions to generate anomalous time series consistent with normal context and well aligned with the given anomaly structure.}
    \label{fig:1}
\end{wrapfigure}

However, alleviating anomaly data scarcity alone is not sufficient for high-quality anomaly generation. Anomalies often exhibit varying morphologies under similar normal contexts \cite{zamanzadeh2024deep}. This variability cannot be adequately captured by context-conditional generation alone, which implicitly assumes a unimodal or smoothly varying anomaly distribution. As a result, models that rely solely on context conditioning tend to average over distinct anomaly patterns, leading to blurred or unrealistic generations. Existing attempts to increase diversity, such as latent perturbation strategies~\cite{darban2025genias}, introduce stochastic variation but do not explicitly model anomaly-specific structure. Consequently, they lack precise control over anomaly morphology and struggle to faithfully represent multimodal anomaly distributions.


In this work, we propose \textbf{CAST}, a
\textbf{C}ontext- and \textbf{A}nomaly \textbf{S}tructure-conditioned
\textbf{T}ime-series anomaly generation framework based on rectified flow matching with a pretraining-finetuning strategy. As illustrated in Figure~\ref{fig:1}, \ourmethod\ incorporates anomaly structure representations as an additional conditioning signal, facilitating more controllable anomaly synthesis while preserving consistency with the surrounding normal behavior. To effectively learn CAST under severe anomaly data scarcity, we adopt a pretraining-finetuning strategy that fully exploits abundant normal data.
In particular, a context-unaware pretraining stage allows the generator to
capture general system dynamics, while subsequent structure-conditioned
finetuning focuses on modeling heterogeneous anomaly morphologies.
Our extensive experiments (see Section~\ref{sec:experiment}) demonstrate that this design yields substantial
improvements in both generation fidelity and downstream task utility.

In summary, our contributions are threefold:

\begin{itemize}
    \item We highlight two key challenges in anomalous time series generation: (i) the scarcity of anomaly data, and (ii) the heterogeneous and multimodal nature of anomaly morphologies, which are not adequately captured by existing time series anomaly generation frameworks.
    

     \item We propose \textbf{CAST}, a context- and anomaly structure-conditioned time series generation framework with a two-stage pretraining–finetuning strategy. The proposed design disentangles normal context and anomaly structure, enabling controllable and diverse anomaly synthesis while effectively leveraging abundant normal data.

    
    \item We conduct extensive experiments on multiple real-world datasets, demonstrating that CAST consistently outperforms state-of-the-art anomaly generation methods in both generation fidelity and downstream anomaly detection performance. Qualitative analyses further reveal that CAST learns a structured and interpretable anomaly latent space, enabling smooth interpolation and synthesis of diverse anomaly morphologies under fixed normal contexts.
\end{itemize}

\section{Preliminaries}
We provide the background required for \ourmethod\ herein. In \ourmethod, we use a vector quantized variational autoencoder (VQ-VAE) as the anomaly structure representation extractor, and a rectified flow matching model as the generator.

\subsection{Vector Quantized Variational Autoencoder}
VQ-VAE consists of an encoder $\text{Enc}(\cdot)$, a quantizer $\text{Quan}(\cdot)$, and a decoder $\text{Dec}(\cdot)$. Given an input signal $x\in \mathbb{R}^d$, the encoder produces a continuous latent representation $z_e = \text{Enc}(x) \in \mathbb{R}^{\tilde{d}}$. Then $z_e$ goes through the quantizer, i.e., it is mapped to the nearest entry in a learnable codebook $\mathcal{C} = \{ e_k \}_{k=1}^{K}$, where $e_k \in \mathbb{R}^{\tilde{d}}$ for all $k\in[K]$.

The quantized latent $z_q$ is obtained via nearest-neighbor lookup, i.e., $z_q = e_{k^\ast}$, $k^\ast = \arg\min_k \| z_e - e_k \|_2.$ The decoder reconstructs the input as $\hat{x} = \text{Dec}(z_q)$. During training, VQ-VAE learns to minimize the loss function 
\begin{align*}
    \mathcal{L}_{\text{VQ-VAE}}
    =
    \| x - \hat{x} \|_2^2
    +
    \| \text{sg}[z_e] - z_q \|_2^2
    +
    \beta \| z_e - \text{sg}[z_q] \|_2^2,
\end{align*}
where $\text{sg}[\cdot]$ denotes the stop-gradient operator. The three terms are reconstruction loss, codebook loss, and commitment loss, respectively. $\beta$ denotes the weight for the commitment loss.

\subsection{Rectified Flow Matching}
Rectified flow learns to transport data samples from a simple base distribution $p_0(x_0)$ to the data distribution $p_{data}(x_1)$. The learning target of rectified flow is a time-dependent vector field $v_\theta(x,t)$, and the samples $x_1\sim p_\theta(x_1)$ can be generated by solving the ordinary differential equation (ODE):
\begin{align*}
    \frac{d x_t}{d t} = v_\theta(x_t, t), \quad x_0\sim p_0(x_0), \quad t \in [0,1].
\end{align*}
The learning target $v_t^*$ of rectified flow is a linear path between $x_0$ and $x_1$, i.e., $x_t^* = (1 - t) x_0 + t x_1$ and $v_t^* = \frac{dx_t^*}{dt}=x_1-x_0$. Therefore, the rectified flow model learns to minimize the objective $\mathcal{L}_{\text{RF}}(\theta) = \mathbb{E}_{x_0,x_1,t}    [\| v_\theta(x_t, t) - v_t^* \|^2]$.

\section{Method}
This section begins by formalizing the time series anomaly generation problem from the probabilistic perspective (Section~\ref{sec:prob_model}). Then we present two core components of our proposed CAST framework, i.e., pretraining-finetuning strategy (Section~\ref{sc:training_pipeline}) and context- and anomaly structure- conditional generation (Section~\ref{sec:scfm}). 
\subsection{Probabilistic Modeling \label{sec:prob_model}}

Let $(\Omega, \mathcal F, \mathbb P)$ be a probability space, and let $X : \Omega \to \mathbb{R}^{T \times C}$ denote a multivariate anomalous time series random variable. We model the anomaly location and duration using a binary mask $M:\Omega \to \{0,1\}^T$. Based on this, an anomalous time series is decomposed into an anomaly segment $X_{\mathrm{ano}} := X\circ M$ and the corresponding normal context $X_{\mathrm{ctx}} := X\circ (\mathbf{1}_T-M)$, where $\mathbf{1}_T$ denotes the all-ones vector and $\circ$ denotes the Hadamard product. A common modeling is that the anomalous segment is generated conditionally on the surrounding normal context,
i.e., $X_{\mathrm{ano}} \sim p(\,\cdot \mid X_{\mathrm{ctx}})$. This formulation captures the dependence of anomaly morphology on the underlying normal dynamics encoded by the context. However, such context-only conditional generation is often insufficient in practice, as anomalies may exhibit substantially different morphologies even under similar normal contexts due to unobserved factors. To address this insufficiency, we introduce a latent-variable formulation for anomaly generation: 
\begin{align*}
    p(X_{\mathrm{ano}} \mid X_{\mathrm{ctx}}) = \int p(X_{\mathrm{ano}} \mid X_{\mathrm{ctx}}, Z=z)\, p(z)\, dz,
\end{align*}
where $Z$ denotes oracle latent which captures anomaly-specific structure that is not explained by the normal context alone. In our framework, a VQ-VAE is used to get an inferred latent $U$, while the conditional generator $p(X_{\mathrm{ano}} \mid X_{\mathrm{ctx}}, U)$ is instantiated using a context- and latent-conditioned flow matching model. Throughout the paper, $X_{\mathrm{ctx}}$ denotes the normal context and $X_{\mathrm{ano}}$ denotes the anomalous segment masked by $M$.
The latent variable $U$ represents the inferred anomaly structure.

\subsection{Pretraining-Finetuning Strategy \label{sc:training_pipeline}}
Our pretraining-finetuning strategy consists of a context-unaware pretraining and a context-aware finetuning.

\textbf{Stage I: Context-unaware Pretraining} First of all, we train a VQ-VAE. Leveraging a trained VQ-VAE, we pretrain the rectified flow matching model without incorporating contextual information. Conditioned on the inferred latent $U$, the model learns the distribution
$p_\theta(X \mid U)$ on a mixed dataset with both normal and anomalous data, aligning the latent with a coherent temporal pattern. This stage exploits abundant normal data to establish stable latent-to-signal mappings and focuses on mode-level generation rather than contextual consistency.
    
\textbf{Stage II: Context-aware Finetuning} We finetune the rectified flow matching generator with structured conditioning on both the normal context $X_{\mathrm{ctx}}$ and the inferred latent code $U$ obtained by a trained VQVAE. The model learns the conditional distribution $p_\theta(X_{\mathrm{ano}} \mid X_{\mathrm{ctx}}, U)$, enabling context-consistent yet structurally diverse anomaly generation.

\subsection{Context- and Anomaly Structure- Conditional Generation }\label{sec:scfm}

\begin{wrapfigure}{r}{0.48\linewidth}
    \centering
    \includegraphics[width=\linewidth]{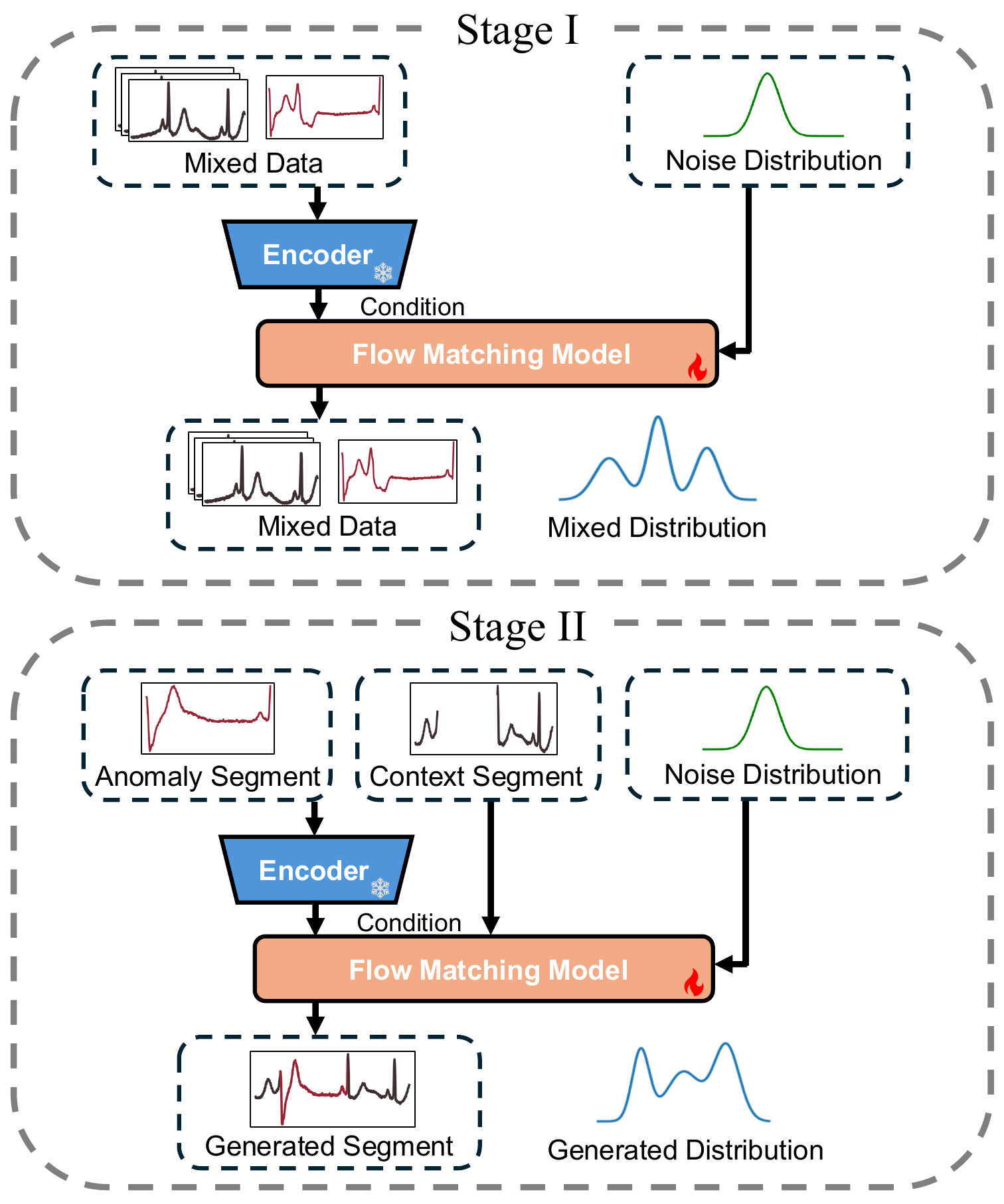}
    \caption{Two-stage training pipeline of CAST.}
    \label{fig:2}
\end{wrapfigure}

Given normal context $X_{\mathrm{ctx}}$ and discrete anomaly-structure latent inferred from real anomaly segments $U$, we model the conditional distribution
of anomaly segments using a structured conditional flow matching framework.
Specifically, we define a conditional time-dependent vector field
$v_\theta(X_t, U, X_{\mathrm{ctx}},t)$ that depends on both
the surrounding normal context and the discrete latent code encoding high-level anomaly structure. During training, we construct a linear interpolation path between noise and the real anomaly segment. Crucially, noise is injected exclusively into the anomalous region, preserving the normal context as a fixed condition:
$X_t
=
X_{\mathrm{ctx}}
+
(1-t)\,(X_0 \circ M)
+
t\,X_{\mathrm{ano}}$.
Under this construction, the target velocity is
$
Y = (X_{\mathrm{ano}} - X_0)\circ M.
$
The conditional vector field is trained by minimizing a masked flow matching objective
that only accounts for the anomalous region: $\mathbb{E}_{X_1,X_0,M,t} [\|v_\theta(X_t, U, X_\mathrm{ctx}, t)-Y\|_2^2]$. The training procedure is summarized in Algorithm~\ref{alg:scfm_train}.
Notably, by restricting noise injection and loss computation to the anomalous segment,
the normal context is preserved as a stable conditioning signal throughout training.

At inference time, anomaly generation is performed by combining
discrete anomaly-structure latents extracted from real anomalous segments
and normal contexts sampled from the normal dataset.
Specifically, a discrete latent code $U$ is inferred from an observed anomaly
segment representing a seen anomaly structure, and a normal context $X_{\mathrm{ctx}}$ is constructed using a mask $M$.
Starting from an initial state where noise is injected only into the masked region,
the conditional flow model integrates from $t=0$ to $t=1$ to generate a synthetic
anomaly that is consistent with the normal context and follows the inferred anomaly structure.
The full inference procedure is given in Algorithm~\ref{alg:vrfm_inference}.

\section{Experiments \label{sec:experiment}}

\begin{table}[t]
\caption{Comparison on five real-world datasets. Metrics are scaled for readability (see row labels). MSE and MAE are computed on anomalous segments only. Lower is better. Negative values in MMD (and KID) may occur due to finite-sample estimation error and should be interpreted as values close to zero.}
\label{tab:all_metrics}
\centering
\begin{tabular}{cc|cccccc}
\toprule
&& \ourmethod & FlowTS & \makecell{Diffusion\\-TS} & TimeVAE & CGATS & GenIAS \\
\midrule

\multirow{5}{*}{MIT-DB}
& MSE ($\times 10^{-3}$)  & \textbf{2.9} & 8.1 & 5.4 & 6.6 & \underline{4.5} & 50.0 \\
& MAE ($\times 10^{-2}$)  & \textbf{3.6} & 6.3 & 4.9 & 5.9 & 4.0 & 17.0 \\
& FID ($\times 10^{-1}$)  & \textbf{0.74} & 1.41 & 9.78 & \underline{1.24} & 4.44 & 32.0 \\
& KID ($\times 10^{-3}$)  & \textbf{0.67} & 2.32 & 21.14 & \underline{1.83} & 2.74 & 60.0 \\
& MMD ($\times 10^{-2}$) & \textbf{1.15} & 3.73 & 9.18 & \underline{1.84} & 8.33 & 13.07 \\
\midrule

\multirow{5}{*}{SVDB}
& MSE ($\times 10^{-3}$)  & \textbf{4.3} & 11.0 & \underline{9.1} & 12.0 & 9.8 & 35.0 \\
& MAE ($\times 10^{-2}$)  & \textbf{3.9} & 6.2 & 6.2 & 7.8 & \underline{5.8} & 14.7 \\
& FID ($\times 10^{-1}$)  & \textbf{2.14} & 3.68 & \underline{2.93} & 4.11 & 5.42 & 21.83 \\
& KID ($\times 10^{-3}$)  & \textbf{0.19} & 1.23 & \underline{1.01} & 2.86 & 3.47 & 32.43 \\
& MMD ($\times 10^{-2}$) & \textbf{0.74} & 1.50 & \underline{1.41} & 1.75 & 2.11 & 10.50 \\
\midrule

\multirow{5}{*}{QTDB}
& MSE ($\times 10^{-3}$)  & \textbf{2.2} & \underline{3.7} & 4.0 & 6.3 & 4.4 & 22.0 \\
& MAE ($\times 10^{-2}$)  & \textbf{2.5} & 3.7 & 4.1 & 4.9 & \underline{3.1} & 11.0 \\
& FID ($\times 10^{0}$)   & \underline{0.64} & 0.91 & 1.03 & 1.98 & \textbf{0.56} & 9.43 \\
& KID ($\times 10^{-3}$)  & \underline{0.23} & 3.45 & 3.80 & 28.65 & \textbf{-0.64} & 162.49 \\
& MMD ($\times 10^{-3}$) & \textbf{4.29} & 5.09 & 5.38 & 5.96 & \underline{4.41} & 7.29 \\
\midrule

\multirow{5}{*}{PV}
& MSE ($\times 10^{-2}$)  & \underline{3.3} & 6.4 & 6.4 & \textbf{3.0} & 4.0 & 16.0 \\
& MAE ($\times 10^{-2}$)  & \textbf{8.8} & 13.0 & 14.0 & 12.0 & \underline{9.9} & 24.0 \\
& FID ($\times 10^{0}$)   & \textbf{1.07} & 8.93 & 12.58 & \underline{2.88} & 2.11 & 3.81 \\
& KID ($\times 10^{-2}$)  & \textbf{0.35} & 16.04 & 22.73 & \underline{3.64} & 2.20 & 5.31 \\
& MMD ($\times 10^{-2}$) & \textbf{0.60} & 1.96 & 2.47 & \underline{1.72} & 1.23 & 1.52 \\
\midrule

\multirow{5}{*}{Metro}
& MSE ($\times 10^{-2}$)  & \textbf{1.4} & 3.2 & 4.4 & 5.2 & \underline{1.8} & 19.0 \\
& MAE ($\times 10^{-2}$)  & \textbf{7.8} & 11.0 & 14.0 & 18.0 & \underline{9.5} & 35.0 \\
& FID ($\times 10^{-1}$)  & \textbf{2.19} & \underline{4.71} & 14.32 & 58.05 & 10.42 & 56.63 \\
& KID ($\times 10^{-3}$)  & \textbf{-7.08} & \underline{-5.10} & 1.94 & 83.32 & -0.03 & 61.29 \\
& MMD ($\times 10^{-3}$) & \textbf{2.06} & \underline{2.88} & 3.79 & 5.89 & 3.55 & 5.42 \\

\bottomrule
\end{tabular}
\end{table}

\subsection{Experimental Setup}
\paragraph{Datasets} The experiments involve five real-world datasets, all of which contain annotated anomalous patterns. We consider three ECG datasets with different sampling frequencies and distinct arrhythmia characteristics to demonstrate the effectiveness of \ourmethod\ under strong anomaly settings. The ECG datasets include the MIT-BIH Arrhythmia Database (MIT-DB),
the QT Database (QT-DB), and the MIT-BIH Supraventricular Arrhythmia Database (SVDB). To further evaluate the effectiveness of \ourmethod\ under weak anomaly settings, we additionally adopt a photovoltaic (PV) active power generation dataset and the UCI Metro Interstate Traffic Volume dataset. We provide detailed dataset information in Appendix~\ref{sec:appendix_dataset}.

\paragraph{Baselines and Implementation Details}
We compare CAST with five representative baseline methods, including two approaches specifically designed for time series anomaly generation (GenIAS and C-GATS), as well as three general-purpose time series generation models (Diffusion-TS, FlowTS, and TimeVAE).
For all general-purpose generators, we adapt their training and inference pipelines to the same context-conditional anomaly generation setting as CAST, where the normal context is provided as a conditioning signal and generation is restricted to masked anomalous regions.
All baselines are trained and evaluated under identical anomaly masks, context construction, and data splits to ensure a fair comparison.
Since GenIAS and C-GATS do not provide publicly available implementations, we re-implement these methods following the architectural designs, training objectives, and hyperparameter settings described in their original papers.
Detailed implementation choices, architectural configurations, and training protocols for all baselines are provided in Appendix~\ref{appendix:implementation_details}.

\paragraph{Metrics}
We evaluate anomalous time series generation from both generation fidelity and downstream utility perspectives. Generation fidelity is assessed with point-wise error including mean squared error (MSE) and mean absolute error (MAE) computed on the anomaly regions only as well as distribution similarity metrics including FID, KID, and MMD. To evaluate practical utility, we adopt a downstream anomaly detection protocol where synthetic anomalies are used to train detectors. Performance is assessed via F1-score on real anomalous time series from the test sets. Unless otherwise specified, all results are averaged over five runs. We evaluate seven anomaly detectors in total, including classical machine learning methods (i.e., Random Forest and CatBoost), deep learning models (i.e., temporal convolutional network (TCN) and RobustTAD~\cite{gao2020robusttad}), a time series foundation model (Moment~\cite{goswami2024moment}), and large language model-based detectors (LTSM~\cite{chuang2024ltsm} and OneFitsAll~\cite{zhou2023one}). Implementation details of these anomaly detectors are given in Appendix~\ref{TSAD_implementation}.

\begin{table}[!t]
    \small
    \renewcommand{\arraystretch}{1.1}
    \caption{Downstream anomaly detection performance (F1-score). Anomaly detectors are trained using synthetic anomalies and evaluated on real anomalous time series in the test sets. The best and the second best results are highlighted in bold and underline, respectively}
    \label{tab:f1_score}
    \centering
    \begin{tabular}{c|ccccccc>{\columncolor{gray!15}}c}
    \toprule
    & 
\makecell{Random\\-Forest} &
\makecell{Cat\\-Boost} &
\makecell{Robust\\-TAD} &
TCN &
Moment &
\makecell{OneFits\\All} & 
LTSM & 
\makecell{Average}
\\
\midrule

\multicolumn{9}{c}{\textbf{MIT-DB}} \\
\midrule
\ourmethod & \textbf{0.905}&\textbf{0.830}&\textbf{0.828}&\textbf{0.653}&\textbf{0.741}&\textbf{0.945}&\textbf{0.919}&\textbf{0.832} \\
FlowTS & 0.655&0.507&0.701&0.607&0.464&\underline{0.910}&0.802&0.664\\
Diffusion-TS & \underline{0.823}&\underline{0.594}&0.732&0.683&0.499&\underline{0.910}&\underline{0.867}&\underline{0.730}\\
TimeVAE &0.575&0.480&0.687&0.682&0.509&0.832&0.740&0.644 \\
C-GATS &0.600&0.368&\underline{0.740}&\underline{0.707}&0.510&0.905&0.831&0.666  \\
GenIAS &0.480&0.425&0.733&0.424&\underline{0.531}&0.754&0.758&0.587  \\
\midrule

\multicolumn{9}{c}{\textbf{SV-DB}} \\
\midrule
\ourmethod & \textbf{0.941}&\textbf{0.943}&\textbf{0.983}&\textbf{0.870}&\textbf{0.911}&\textbf{0.979}&\textbf{0.917}&\textbf{0.935} \\
FlowTS & \underline{0.886}&\underline{0.819}&0.831&0.794&\underline{0.817}&\underline{0.974}&0.289&\underline{0.773} \\
Diffusion-TS & 0.825&0.757&0.918&\underline{0.824}&0.719&0.963&0.306&0.759  \\
TimeVAE & 0.746&0.763&\underline{0.921}&0.777&0.696&0.976&\underline{0.395}&0.753 \\
C-GATS & 0.777 & 0.59 & 0.663 & 0.098 & 0.617 & 0.843&0.223&0.544 \\
GenIAS & 0.620 & 0.434 & 0.063 & 0.657 & 0.268 & 0.832&0.382&0.465 \\
\midrule

\multicolumn{9}{c}{\textbf{QT-DB}} \\
\midrule
\ourmethod & \textbf{0.950} & \textbf{0.951} & \textbf{0.959} & \textbf{0.635} & \textbf{0.944} & \underline{0.975} & \textbf{0.861} & \textbf{0.896}\\
FlowTS & \underline{0.932} & \underline{0.915} & 0.938 & \underline{0.604} & \underline{0.837} & \textbf{0.977} & \underline{0.667} & \underline{0.839} \\
Diffusion-TS & 0.838 & 0.845 & 0.934 & 0.581 & 0.647 & 0.953 & 0.634 & 0.776 \\
TimeVAE & 0.654 & 0.455 & 0.583 & 0.548 & 0.497 & 0.932 & 0.503 & 0.596 \\
C-GATS & 0.913 & 0.893 & \underline{0.944} & 0.591 & 0.778 & 0.866 & 0.545 & 0.790\\
GenIAS & 0.475 & 0.374 & 0.000 & 0.205 & 0.29 & 0.92 & 0.505 & 0.396 \\
\midrule

\multicolumn{9}{c}{\textbf{PV}} \\
\midrule
\ourmethod & \textbf{0.225} & \textbf{0.125} & \textbf{0.607} & \underline{0.070} & \underline{0.369} & \textbf{0.432} & \underline{0.420} & \textbf{0.321}\\
FlowTS & 0.106 & 0.037 & 0.000 & 0.000 & \textbf{0.392} & 0.293 & 0.419 & \underline{0.178}\\
Diffusion-TS & 0.066 & 0.069 & 0.097 & 0.001 & 0.349 & 0.174 & 0.417 & 0.168 \\
TimeVAE &0.036&0.003&0.046&0.050&0.286&0.172 & 0.419 & 0.145 \\
C-GATS & 0.002 & 0.009&0.058&0.000&0.366&\underline{0.336} & \textbf{0.472} & 0.178\\
GenIAS & 0.000 & 0.000 & 0.000& \textbf{0.075}&0.233&0.170&0.400&0.125 \\
\midrule

\multicolumn{9}{c}{\textbf{Metro}} \\
\midrule
\ourmethod & \underline{0.662}&0.096&\textbf{0.677}&\textbf{0.497}&\textbf{0.629}&\underline{0.698} & \textbf{0.769} & \textbf{0.575}\\
FlowTS & \textbf{0.673}&\underline{0.128}&0.388&\underline{0.435}&\underline{0.615}&\textbf{0.700} & 0.755 & \underline{0.528}\\
Diffusion-TS & 0.468&\textbf{0.230}&0.257&0.362&0.603&0.606 & \underline{0.762} & 0.470\\
TimeVAE & 0.068&0.110&0.000&0.000&0.341&0.430 & 0.647 & 0.228 \\
C-GATS & 0.354&0.096&\underline{0.435}&0.332&0.562&0.602 & 0.650 & 0.433\\
GenIAS & 0.000 & 0.000 & 0.000 & 0.000 & 0.204 & 0.459 & 0.602 & 0.181 \\
\bottomrule
\end{tabular}   
\end{table}

\subsection{Main Results}

Table~\ref{tab:all_metrics} reports both point-wise fidelity (MSE, MAE) and distribution matching metrics (FID, KID, MMD) across five real-world datasets. \ourmethod\ consistently achieves the best MSE and MAE on almost all datasets. \ourmethod\ also demonstrates strong performance in distribution matching, achieving the best results in almost all cases across FID, KID, and CMMD. Notably, the improvements are consistent under both strong-anomaly settings (MIT-DB, SVDB, and QTDB) and weak-anomaly settings (PV and Metro), indicating that \ourmethod\ is robust to varying anomaly characteristics and generalizes well across domains. Table~\ref{tab:f1_score} shows the utility of generated anomalies for downstream anomaly detection by showing F1-scores on the test set. Across all datasets and detection models,
\ourmethod\ consistently yields the highest average F1-scores,
often with a clear margin over baseline methods. The performance gains observed for different detectors suggest the benefits of our generated anomalies are not due to the bias of the detector but exhibit a better distributional similarity.

\subsection{Ablation Study}

\begin{wraptable}{r}{0.48\textwidth}
\centering
\setlength{\tabcolsep}{3pt}
\caption{Ablation study of model components. ASC denotes anomaly structure conditioning, and CUP denotes context-unaware pretraining. The best metrics are highlighted in bold.}
\label{tab:ablation_study}
\begin{tabular}{c ccc ccc}
\toprule
&\multicolumn{3}{c}{Components} &
\multicolumn{3}{c}{Metrics} \\
\cmidrule(lr){2-4} \cmidrule(lr){5-7}
&Main & ASC & CUP & MSE$\downarrow$ & MAE$\downarrow$ & F1$\uparrow$ \\
\midrule
(a)&$\checkmark$ &        &        &0.023&0.080&0.595\\
(b)&$\checkmark$ & $\checkmark$ &        &0.040&0.080&0.599 \\
(c)&$\checkmark$ &        & $\checkmark$ &0.021&0.072&0.686 \\
(d)&$\checkmark$ & $\checkmark$ & $\checkmark$ & \textbf{0.011} & \textbf{0.053} & \textbf{0.711} \\
\bottomrule
\end{tabular}
\end{wraptable}

Table~\ref{tab:ablation_study} presents an ablation study to analyze the effectiveness of anomaly structure conditioning (ASC) and context-unaware pretraining (CUP) in \ourmethod. Comparing (a) and (b), introducing ASC without pretraining fails to improve the performance. This indicates that discrete anomaly structures are difficult to exploit when the generator is directly trained under anomaly scarcity. Comparing (a) and (c), CUP consistently improves all metrics. This demonstrates that leveraging abundant normal data before learning context-dependent anomaly generation is beneficial. Comparing (c) and (d), adding ASC on top of CUP yields further improvements. It shows that the flow matching model learns to exploit the anomaly representation with ASC, hence leading to better generation fidelity and downstream utility. Overall, the best performance is achieved when ASC and CUP are combined, validating the necessity of modeling anomaly generation as a context- and anomaly structure- conditioned framework and adopting the proposed pretraining-finetuning strategy.


\begin{figure}[bht]
\centering

\begin{subfigure}[t]{0.49\linewidth}
    \centering
    \includegraphics[width=\linewidth]{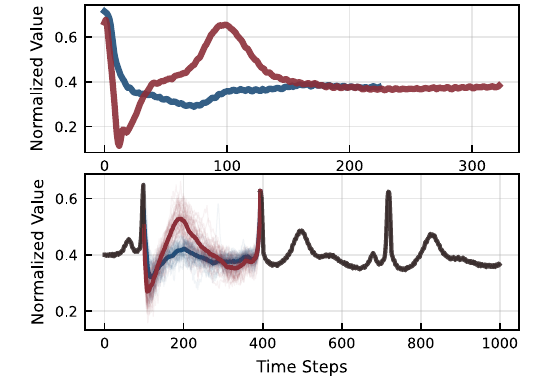}
    \caption{Same context; Different anomaly structure latent}
    \label{case_study_left}
\end{subfigure}
\hfill
\begin{subfigure}[t]{0.49\linewidth}
    \centering
    \includegraphics[width=\linewidth]{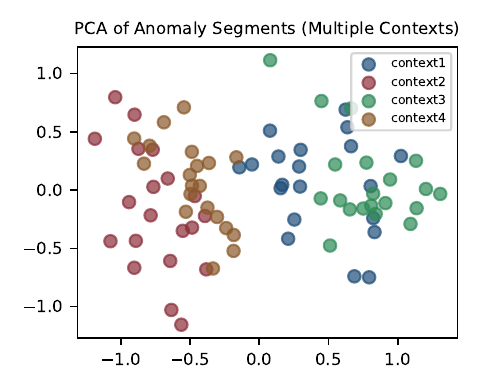}
    \caption{Same anomaly structure latent; Different contexts.}
    \label{case_study_right}
\end{subfigure}

\caption{Controllability of anomaly generation. Left (\ref{case_study_left}): varying anomaly structure while fixing the normal context produces diverse anomaly patterns. The upper panel shows two anomaly segments. The lower panel shows generated time series conditioned on the same normal context but on these two different anomaly structures. Right (\ref{case_study_right}): varying the normal context while fixing the anomaly structure preserves the anomaly pattern but adapts it to different contexts.}
\label{fig:case_study}
\end{figure}

\subsection{Qualitative Analysis}


\begin{figure}[t]
  \begin{center}
\centerline{\includegraphics[width=0.7\linewidth]{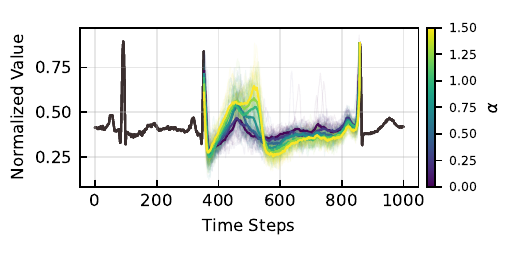}}
\caption{
Generated anomalies conditioned on the same context with interpolated latent codes $u_{\alpha} = \alpha u_1 + (1-\alpha)u_0$, where $\alpha \in \{0.0, 0.25, 0.5, 0.75, 1.0, 1.25, 1.5\}$ denotes the interpolation coefficient; $u_0$ and $u_1$ are two inferred latent variables from seen anomalies. Anomaly morphology varies smoothly with latent interpolation $\alpha$,
while the normal context remains unchanged.
}
\label{fig:interpolation_std_mask}
  \end{center}
\end{figure}
Figure~\ref{fig:case_study} illustrates how \ourmethod\ controls anomaly generation through disentangled representations of normal context and anomaly structure. In Figure~\ref{case_study_left}, we fix the normal context while varying the anomaly-structure latent. We observe that the generated time series exhibit distinct anomaly morphologies: the red curve shows a pronounced rising–falling pattern with a clear peak, whereas the blue curve exhibits a smoother and flatter structure. This demonstrates that \ourmethod\ effectively leverages the learned anomaly-structure latent to produce diverse yet coherent anomaly patterns under the same context. Importantly, the generated anomalies do not simply replicate the given anomaly segments, but instead adapt them to the surrounding normal context while preserving their structural characteristics. In Figure~\ref{case_study_right}, we instead fix the anomaly-structure latent and vary the normal context. The PCA visualization shows that anomaly segments generated under different contexts form distinct clusters, corresponding to different contextual conditions. 
This suggests that \ourmethod\ preserves the semantic structure of anomalies while still remaining responsive to the normal context, rather than ignoring it.

Figure~\ref{fig:interpolation_std_mask} further investigates the structure of the learned latent space by interpolating between two anomaly-inferred latent codes, i.e., $u_{\alpha} = \alpha u_1 + (1-\alpha)u_0$. As the interpolation coefficient $\alpha$ varies, the generated anomalies exhibit smooth and continuous changes in morphology, while the surrounding normal context remains unchanged. This behavior indicates that the flow matching model can capture the semantic meaning of anomaly structure latent, rather than merely memorizing isolated patterns. It indicates that \ourmethod\ has the potential to synthesize novel anomaly morphologies beyond those observed during training.

\section{Related Works}
\subsection{Time Series Anomaly Generation}



Time series anomaly generation aims to synthesize anomalous samples to alleviate label scarcity and improve downstream performance. Early work primarily relies on rule-based anomaly injection methods \cite{yun2019cutmix,zhang2017mixup,turowski2022modeling}, which are simple and efficient but offer limited diversity and fail to capture the complex distribution of real anomalies.

Recent learning-based approaches mainly follow GAN- and VAE-based paradigms \cite{zhao2024sgad,9378139,singh2022c,darban2025genias}. GAN-based methods typically treat generation as an auxiliary component for anomaly detection, rather than explicitly modeling anomaly distributions, and often lack fine-grained controllability. VAE-based methods, such as GenIAS \cite{darban2025genias} and C-GATS \cite{singh2022c}, attempt to mitigate anomaly scarcity via latent perturbation or pretraining–finetuning strategies, but their generation quality is constrained by limited expressiveness and incomplete modeling of anomaly distributions.

\subsection{Time Series Anomaly Detection}
\citet{liu2024elephant} groups time series anomaly detection (TSAD) algorithms into three categories, unsupervised, semi-supervised, and supervised. Unsupervised methods \cite{liu2008isolation,breunig2000lof, paffenroth2018robust,goldstein2012histogram,boniol2021unsupervised,yeh2018time} detect anomalies as deviations from global or local data distributions without requiring labeled anomaly data. Similar to unsupervised methods, Semi-supervised methods works well under anomaly data scarcity \cite{su2019robust,audibert2020usad,malhotra2015long,sakurada2014anomaly}, learning normal temporal dynamics during training and identify anomalies as deviations at test time. 



Supervised TSAD methods rely on labeled normal and anomalous data \cite{gao2020robusttad,ren2019time,carmona2021neural}, but their effectiveness is often limited by the scarcity of anomaly annotations. To alleviate this issue, recent work explores large-scale representation learning, including time series foundation models \cite{yue2022ts2vec,tonekaboni2021unsupervised,wu2022timesnet,goswami2024moment} and large language models \cite{sui2025training,zhou2023one,chuang2024ltsm}. Despite strong representation power, supervised anomaly detection with these models remains constrained by limited anomalous data. Our work directly addresses this limitation by generating realistic and controllable anomalous time series to support data augmentation for supervised and representation-based TSAD.

\subsection{General Time Series Generation}

Early time series generation methods were primarily based on VAEs and GANs. VAE-based approaches emphasize training stability and structured latent representations, including models with explicit temporal decomposition \cite{desai2021timevae} and latent linear dynamics inspired by Koopman theory \cite{naiman2024generative}. In contrast, GAN-based methods focus on matching complex temporal distributions, with representative works incorporating supervised embedding losses \cite{yoon2019time}, causal optimal transport constraints \cite{xu2020cot}, and self-attention mechanisms \cite{jeha2022psa}. Despite strong empirical performance, GAN-based models often suffer from mode collapse and limited controllability.


Diffusion and flow matching have recently become prominent approaches for time series generation due to their stable training dynamics. Diffusion-based methods have explored diverse conditioning and architectural designs, including transformer-based diffusion \cite{yuan2024diffusionts}, structured state-space guidance \cite{kollovieh2023predict}, autoregressive and mixup conditioning \cite{shen2023non}, vision-inspired diffusion \cite{naiman2024utilizing}, and masked diffusion for irregular time series \cite{fadlon2025diffusion}. Despite strong empirical performance, diffusion models remain computationally expensive. Flow matching provides a more efficient alternative, as demonstrated by FlowTS \cite{hu2024flowts} and subsequent extensions \cite{zhang2025towards}.

\section{Conclusion}
In this work, we study the problem of anomalous time series generation. We observe that the challenges of anomalous time series generation stem both from data scarcity and anomaly morphology variability. To address the challenges, we propose \ourmethod, a \textbf{c}ontext- and \textbf{a}nomaly \textbf{s}tructure-conditioned \textbf{t}ime series generation framework with a principled pretraining-finetuning strategy.
The pretraining-finetuning strategy can take advantage of abundant normal data. Moreover, conditioning on normal context and anomaly structure enables generation with controllable morphologies. Extensive experiments on multiple benchmarks demonstrate that CAST consistently outperforms existing anomaly generation methods, with a 26\% decrease in MSE and 17\% decrease in MAE regarding generation fidelity and a 16\% increase in F1-score regarding anomaly detection compared to baselines. Future work includes extending the framework to incorporating additional data modalities, and exploring its application to downstream tasks. Beyond quantitative improvements, qualitative analyses suggest that CAST can learn and use a semantically organized anomaly latent space, where continuous latent manipulations give rise to consistent morphological variations. Such structure supports flexible recombination of anomaly patterns beyond the observed anomaly patterns.

\section*{Acknowledgment}
This research was partially funded by the National Institutes of Health (NIH) under award 1OT2OD038051. The views and conclusions contained in this document are those of the authors and should not be
interpreted as representing the official policies, either expressed or implied, of the NIH.

\newpage
\bibliography{999_reference}
\bibliographystyle{style/icml2025}

\titlespacing*{\section}{0pt}{*1}{*1}
\titlespacing*{\subsection}{0pt}{*1.25}{*1.25}
\titlespacing*{\subsubsection}{0pt}{*1.5}{*1.5}

\setlength{\abovedisplayskip}{\baselineskip}
\setlength{\abovedisplayshortskip}{0.5\baselineskip}
\setlength{\belowdisplayskip}{\baselineskip}
\setlength{\belowdisplayshortskip}{0.5\baselineskip}

\clearpage
\appendix
\label{sec:append}
\part*{Appendix}
{
\setlength{\parskip}{-0em}
\startcontents[sections]
\printcontents[sections]{ }{1}{}
}

\setlength{\parskip}{.5em}
\section{Datasets \label{sec:appendix_dataset}}
All five datasets we adopted herein are open to public. Table~\ref{tab:dataset_statistics} shows the statistics of all five datasets. Below we provide the link to download the datasets:
\begin{itemize}
    \item MIT-BIH Arrhythmia Database (MIT-DB): 
    
    \url{https://physionet.org/content/mitdb/1.0.0/}
    \item QT Database (QT-DB): 
    
    \url{https://physionet.org/content/qtdb/1.0.0/}
    \item MIT-BIH Supraventricular Arrhythmia Database (SVDB): 
    
    \url{https://physionet.org/content/svdb/1.0.0/}
    \item photovoltaic (PV) active power generation: 
    \begin{enumerate}
        \item \url{https://www.solar.sheffield.ac.uk/pvlive/}
        \item \url{https://pages.nist.gov/netzero/data.html#download_data} 
    \end{enumerate}
    \item UCI Metro Interstate Traffic Volume dataset:
    
    \url{https://archive.ics.uci.edu/dataset/492/metro+interstate+traffic+volume}
\end{itemize}

\begin{table}[!ht]
    \centering
    \caption{Statistics of the time series datasets used in our experiments.}
    \label{tab:dataset_statistics}
    \begin{tabular}{lcccc}
        \toprule
        Dataset
        & Dimension
        & \makecell{Sequence\\ Length}
        & \makecell{Minimum \\Anomaly Length}
        & \makecell{Maximum \\Anomaly Length} \\
        \midrule
        MIT-DB & 2  & 1000 & 180  & 800 \\
        QT-DB & 2  & 600  & 80  & 450 \\
        SV-DB & 2  & 800 & 30  & 360 \\
        PV & 1 & 200 & 20 & 144 \\
        Metro & 1 & 72 & 24 & 24 \\
        
        \bottomrule
    \end{tabular}
\end{table}




\section{More Implementation Details \label{appendix:implementation_details}}
All experiments in this paper can run on a NVIDIA RTX A6000 with 48GB GPU memory. The training time of each method takes less than 24 hours.

\ourmethod\ adopts an encoder--decoder Transformer architecture. The base architecture is adapted from FlowTS~\cite{hu2024flowts}. 
\ourmethod\ incorporates three types of conditioning: time conditioning, contextual conditioning, and anomaly structure conditioning. Time conditioning is implemented via adaptive normalization, where the continuous time variable is embedded and injected into each Transformer block to modulate the hidden representations.
Contextual conditioning is achieved by masking the input sequence and restricting the flow matching objective to the corrupted regions, enabling the model to learn context-aware generation and imputation.
Anomaly structure conditioning is realized by conditioning the velocity field on discrete latent representations extracted from a pretrained vector-quantized encoder. The latent representation is injected to the model via adaptive normalization as well. 

For all datasets, VQ-VAE in \ourmethod\ is trained on mixed normal and anomalous segments with a codebook size of 500. The encoder maps each time series into a sequence of 4 discrete codes with embedding dimension 8, which are later used as fixed anomaly-structure representations.The flow model is parameterized by an encoder--decoder Transformer with $4$ encoder layers and $4$ decoder layers. All Transformer blocks use a hidden dimension of $64$ with $4$ attention heads and a feed-forward expansion ratio of $4$. The pesudo code for \ourmethod's\ finetuning and inference are given in Algorithm~\ref{alg:scfm_train} and Algorithm~\ref{alg:vrfm_inference}, respectively. \ourmethod is pretrained for up to 100 epochs and tuned for up to 500 epochs with early stop.

We use the codebases below to implement the corresponding baselines
\begin{enumerate}
    \item \textbf{FlowTS}: https://github.com/UNITES-Lab/FlowTS.git
    \item \textbf{Diffusion-TS}: https://github.com/Y-debug-sys/Diffusion-TS.git
    \item \textbf{TimeVAE}: https://github.com/wangyz1999/timeVAE-pytorch.git
\end{enumerate}
To optimize their performance in our experiment setting, these three models are trained to do context-conditional anomalous time series generation, i.e., imputation. For fair comparison, FlowTS and Diffusion-TS adopt the same model size as \ourmethod, i.e., an encoder--decoder Transformer with $4$ encoder layers and $4$ decoder layers, where each Transformer block uses a hidden dimension of $64$ with $4$ attention heads and a feed-forward expansion ratio of $4$. 
TimeVAE uses a convolutional variational autoencoder with explicit trend and seasonal decomposition. The encoder consists of three convolutional layers with hidden sizes $\{50,100,200\}$ and maps an input sequence into a latent space of dimension $64$. The decoder reconstructs the signal by combining a level component, a third-order polynomial trend component, multiple seasonal components with custom periodicities, and a residual convolutional decoder.
The KL divergence term is weighted by $10^{-4}$ during training.

We implement GenIAS using a perturbation-based VAE. GenIAS uses the same architecture as TimeVAE. The encoder consists of convolutional layers with hidden sizes $\{50,100,200\}$ and maps the input into a latent space of dimension $64$.
The decoder reconstructs the signal by combining a third-order polynomial trend component and multiple seasonal components with custom periodicities. 
During training, we set $\delta_{min}=0.01$, $\delta_{max}=0.02$, and $\sigma_{prior}=0.5$. The weights for KL loss and perturbation loss are $10^{-4}$ and 0.1, respectively. 

To implement C-GATS, we implement a CNN-VAE with a 1D ResNet-style encoder--decoder architecture. The encoder consists of four residual convolutional stages with channel sizes $\{64,64,64,64\}$ and outputs a sequence of latent variables of length $8$ with embedding dimension 64. The decoder mirrors this structure using residual upsampling blocks with channel sizes $\{64,64,32,32,16,16\}$ and reconstructs the signal to the original temporal resolution. KL divergence term is weighted by $10^{-4}$.

Hyperparameters of all baselines have been tuned maticulously for the optimal performance of baselines. All baselines are trained for up to 2000 epochs with early stop. All models are trained using the Adam optimizer with an initial learning rate $10^{-4}$.
We apply a ReduceLROnPlateau learning rate scheduler that monitors the training loss and reduces the learning rate by a factor of $0.8$ if no improvement larger than $10^{-4}$ is observed for one epoch.
The learning rate is lower-bounded by $10^{-5}$ to ensure stable convergence. All relevant code will be released upon the acceptance of the paper.

    

\begin{algorithm}[!ht]

\caption{CAST Finetuning}

\label{alg:scfm_train}

\begin{algorithmic}[1]

\State \textbf{Input:} anomalous time series dataset $\mathcal{D}_a$, a well-trained latent encoder $E(\cdot)$

\State \textbf{Initialize:} flow matching model $v_\theta(\cdot)$

\While{not converged}

    \State Sample anomalous time series $(x_1, m) \sim \mathcal{D}_a$
    \State Sample noise $x_0 \sim p_0$ and time $t \sim \mathcal{U}(0,1)$
    \State $x_{\mathrm{ano}} = x_1 \circ m$

    \State $x_{\mathrm{ctx}} = x_1 \circ (1 - m)$
    \State $u = E(x_{\mathrm{ano}})$
    \State $x_t = x_{\mathrm{ctx}} + (1 - t)x_0 \circ m + t x_{\mathrm{ano}}$
    \State $y = (x_{\mathrm{ano}} - x_0) \circ m$ \Comment{Compute target velocity}
    \State Update $\theta$ by minimizing:

    \State $\displaystyle
        \left\|
        v_\theta(x_t, u, x_{\mathrm{ctx}}, t)
        - y
        \right\|_2^2
    $
\EndWhile
\end{algorithmic}
\end{algorithm}

\begin{algorithm}[!ht]
\caption{CAST Inference}
\label{alg:vrfm_inference}
\begin{algorithmic}[1]
\State \textbf{Input:} conditional flow matching $v_\theta(\cdot)$, discrete encoder $E(\cdot)$, source distribution $p_0$, predefined maske distribution $P_m$, anomaly dataset $\mathcal{D}_a$, normal dataset $\mathcal{D}_n$
\State sample anomalous time series $(\tilde{x}, \tilde{m})\sim\mathcal{D}_a $
\State construct $\tilde{x}_{\mathrm{ano}}=\tilde{x}\circ \tilde{m}$
\State get latent $u = E(\tilde{x})$

\State sample normal time series $x \sim \mathcal{D}_n$
\State sample anomaly mask $m\sim P_m$
\State construct normal context $x_{\mathrm{ctx}}=x\circ (1-m)$
\State sample noise $\varepsilon \sim p_0$
\State construct $x_0=x_{\mathrm{ctx}}+\varepsilon\circ m$
\State ODE integrate $x_0$ from $t = 0$ to $t = 1$ using velocity field $v_\theta(x_t, u, x_{\mathrm{ctx}},t)$
\end{algorithmic}
\end{algorithm}

\section{Anomaly Detector Implementation Details\label{TSAD_implementation}}

This section provides implementation details of the anomaly detectors used in our experiments.
All detectors perform \emph{point-wise anomaly detection}: given an input multivariate time series of length $T$, the model predicts a binary anomaly label for each timestep.
Unless otherwise specified, detectors are trained on generated data and evaluated on real data. All experiments are repeated five times with different random seeds, and the average results are reported.

\paragraph{Random Forest.}
The Random Forest (RF) detector is implemented using the \texttt{sklearn} \texttt{RandomForestClassifier}.
For each timestep, a fixed-length temporal window of size $2w{+}1$ ($w=\texttt{feat\_window\_size}$) is extracted with reflection padding at sequence boundaries.
From each window, we construct a feature vector by concatenating:
(1) raw values within the window,
(2) per-channel mean,
(3) per-channel standard deviation,
and (4) mean of first-order temporal differences.
The resulting features are flattened and used for point-wise classification.
The RF model uses 300 trees, no depth limit, a minimum leaf size of 10, and balanced class weights.

\paragraph{CatBoost.}
The CatBoost detector uses the same window-based feature extraction as the Random Forest baseline.
We employ a \texttt{CatBoostClassifier} with 500 boosting iterations, tree depth 8, learning rate 0.05, and the \texttt{Logloss} objective.
Class imbalance is handled via automatic class weight balancing.

\paragraph{Robust-TAD.}
Robust-TAD is implemented as a U-Net–style fully convolutional network for time series anomaly detection.
The encoder consists of stacked 1D convolutional blocks with channel sizes $\{16, 32, 64, 128, 256\}$ and max-pooling for temporal downsampling.
The decoder mirrors the encoder using transposed convolutions and skip connections to recover temporal resolution.
The model outputs per-timestep anomaly logits.
Training uses a weighted binary cross-entropy loss with logits, where the positive class weight is set according to the anomaly ratio in the training data.
Optimization is performed with Adam, together with learning rate scheduling based on validation loss and early stopping.

\paragraph{TCN.}
The Temporal Convolutional Network (TCN) detector consists of a stack of dilated causal convolutional blocks with exponentially increasing dilation factors.
We use six layers with kernel size 3 and a hidden dimension of 128.
A $1\times1$ convolutional head maps the hidden representations to per-timestep anomaly logits.
Training follows the same protocol as Robust-TAD, using weighted binary cross-entropy loss, Adam optimization, learning rate scheduling, and early stopping.

\paragraph{MOMENT.}
For the MOMENT-based detector, we adopt a pretrained MOMENT time series foundation model as a frozen backbone and train a lightweight classification head.
Input sequences are resampled to a fixed length of 512 to match the pretrained model requirements.
Patch-level embeddings produced by MOMENT are mapped to per-timestep anomaly logits via a multilayer perceptron head.
We use cross-entropy loss for training and perform linear probing, keeping the backbone frozen.

\paragraph{OneFitsAll.}
The OneFitsAll detector is implemented using a GPT-based time series model.
Input sequences are patchified with stride 1 and embedded before being processed by the first six layers of a pretrained GPT-2 backbone.
Only layer normalization and MLP parameters are fine-tuned, while the remaining parameters are frozen.
A linear head predicts per-timestep anomaly logits.
Training uses weighted binary cross-entropy loss with logits.
Learning rate scheduling and early stopping are based on validation F1 score.

\paragraph{LTSM.}
The LTSM detector is an LLM-based time series model with patch tokenization.
Input sequences are normalized per sample and divided into overlapping patches for each variable channel.
Patch embeddings are projected to the hidden dimension of a pretrained transformer backbone, where only the first few layers are retained.
For anomaly detection, hidden representations are aggregated across variable channels and mapped to per-timestep anomaly logits via a linear head.
Training minimizes point-wise binary cross-entropy with logits.
Training and evaluation are implemented using the HuggingFace \texttt{Trainer}, with optional early stopping and selection of the best checkpoint based on evaluation loss.

\section{Theoretical Analysis \label{app:theory}}

In this section, we present a theoretical analysis to shed light on the intuition behind our principled conditioning design. We emphasize that the goal is to provide conceptual insight, rather than to establish a fundamentally novel theoretical contribution. The focus of this work is still empirical.

\paragraph{Notations and Definitions}
Let $(\Omega,\mathcal F,\mathbb P)$ be a probability space.
We consider random variables
$X_0 : \Omega \to \mathcal{X}_{\mathrm{ano}}$ and
$(X_1, X_{\mathrm{ctx}}, Z) : \Omega \to \mathcal{X}_{\mathrm{ano}} \times \mathcal{X}_{\mathrm{ctx}} \times \mathcal Z$,
where $X_0 \sim \pi_0$ and $(X_1, X_{\mathrm{ctx}}, Z) \sim P_1$.
Here, $X_0$ denotes a reference random variable drawn from the base distribution $\pi_0$ (e.g., standard Gaussian distribution),
$X_1$ denotes an anomalous time series,
$X_{\mathrm{ctx}}$ represents the corresponding normal context,
and $Z$ is an oracle latent variable characterizing anomaly structure.

For any pair of random variables $(A,B)$, conditional variance is defined as
\begin{align*}
\mathrm{Var}(A \mid B)
\;:=\;
\mathbb E\!\left[\|A-\mathbb E[A\mid B]\|^2 \,\middle|\, B\right],
\end{align*}
which is a $B$-measurable random variable. Define the target random variable $Y := X_1 - X_0$.
For each $t\in[0,1]$, we define the interpolated random variable $X_t := (1-t)X_0 + t X_1$, and let $W_t := (X_t, X_{\mathrm{ctx}})$ denote the observable conditioning information.
We further define the $\sigma$-algebras $\mathcal F_t := \sigma(W_t, Z)$ and $\mathcal G_t := \sigma(W_t, U)$, where $\sigma(\cdot)$ denotes the $\sigma$-algebra generated by the enclosed random variables,
and $U$ denotes an inferred latent representation. Following the flow matching framework, we consider an ordinary differential equation
defined on $t\in[0,1]$: 
\begin{align*}
\mathrm d S_t = v_c(S_t, X_{\mathrm{ctx}}, C, t)\,\mathrm dt,
\end{align*} 
where $S_t$ is a state random variable,$v_c(\cdot)$ is a conditional vector field,
and $C$ denotes a generic conditioning variable other than normal context $X_{\mathrm{ctx}}$. Given a conditioning variable $C$, the conditional vector field
$v_c$ is learned by minimizing the flow matching objective
\begin{align*}
\mathcal{L}_0(v_c)
\;=\;
\int_0^1
\mathbb E\!\left[
\| Y - v_c(X_t, X_{\mathrm{ctx}}, C, t) \|^2
\right]\mathrm dt .
\end{align*}
The population minimizer satisfies
\begin{align*}
v_c^*(X_t, X_{\mathrm{ctx}}, C, t)
\;=\;
\mathbb E\!\left[Y \mid W_t, C\right].
\end{align*}

Consequently, the optimal conditional flow matching objective is 
\begin{align*}
\mathcal{L}_0(v_c^*)
=
\int_0^1
\mathbb E\!\left[
\mathrm{Var}(Y \mid W_t, C)
\right]\mathrm dt .
\end{align*}

\begin{assumption}[Oracle Sufficiency]\label{ass:oracle_suff}
For each $t \in [0,1]$, we assume that
\[
Y \perp U \mid \mathcal F_t,
\]
where $\perp$ denotes conditional independence.
That is, conditioned on the observable information $W_t$ and the oracle latent
variable $Z$, the inferred representation $U$ does not provide any additional
predictive information about $Y$.

Equivalently, the following identities hold:
\begin{align*}
    \mathbb E\!\left[ Y \mid W_t, U, Z \right]
    &= \mathbb E\!\left[ Y \mid W_t, Z \right], \\
    \mathrm{Var}\!\left( Y \mid W_t, U, Z \right)
    &= \mathrm{Var}\!\left( Y \mid W_t, Z \right).
\end{align*}
\end{assumption}

\begin{assumption}[Smooth Latent Space]\label{ass:lipschitz_latent}
Let $W_t = (X_t, X_{\mathrm{ctx}})$ and define the conditional mean function
\[
m_t(w,z) := \mathbb E\!\left[ Y \mid W_t = w,\, Z = z \right],
\]
where $w$ belongs to the support of $W_t$.
We assume that there exists a constant $L > 0$ such that for all $t \in [0,1]$,
for almost every $w$, and for all $z, z'$,
\[
\| m_t(w,z) - m_t(w,z') \| \le L \, \| z - z' \|.
\]
\end{assumption}


\begin{lemma}[Optimal Conditional Flow Matching Objective Reduction under Inferred Latents]\label{thm:gain_U}
Let $X_0\sim\pi_0$ and $(X_1,X_{\mathrm{ctx}})\sim P_1$ be random variables, and define
$Y:=X_1-X_0$. For each $t\in[0,1]$, define the interpolant
$X_t:=(1-t)X_0+tX_1$ and the joint conditioning random variable
$W_t:=(X_t,X_{\mathrm{ctx}})$.
Let $v^*$ denote the population minimizer of the conditional flow matching objective
conditioned on $W_t$ (i.e., $v^*(X_t,X_{\mathrm{ctx}},t)=\mathbb E[Y\mid W_t,t]$), and let
$v_u^*$ denote the population minimizer conditioned on $(W_t,U)$ (i.e.,
$v_u^*(X_t,X_{\mathrm{ctx}},U,t)=\mathbb E[Y\mid W_t,U,t]$).
Then the optimal risks satisfy the exact decomposition
\begin{align}
\mathcal L_0(v_u^*)
=
\mathcal L_0(v^*)
-
\int_0^1 
\mathbb E\!\left[
\mathrm{Var}\!\left(\mathbb E[Y\mid W_t,U]\mid W_t\right)
\right] dt .
\end{align}
In particular, conditioning on $U$ is guaranteed to not increase the optimal flow matching objective, i.e., $\mathcal L_0(v_u^*)\le \mathcal L_0(v^*)$.
\end{lemma}

\begin{proof}
By definition of the optimal flow matching objective under squared loss,
\[
\mathcal L_0(v^*)=\int_0^1 \mathbb E\!\left[\mathrm{Var}(Y\mid W_t)\right]dt,
\qquad
\mathcal L_0(v_u^*)=\int_0^1 \mathbb E\!\left[\mathrm{Var}(Y\mid W_t,U)\right]dt.
\]
For any fixed $t$, applying the conditional law of total variance to $Y$ given $W_t$ yields
\begin{align}
\mathrm{Var}(Y\mid W_t)
=
\mathbb E\!\left[\mathrm{Var}(Y\mid W_t,U)\mid W_t\right]
+
\mathrm{Var}\!\left(\mathbb E[Y\mid W_t,U]\mid W_t\right).
\label{eq:ctv_gainU}
\end{align}
Taking expectation of~\eqref{eq:ctv_gainU} over $W_t$ and integrating over $t\in[0,1]$, we obtain
\begin{align}
\mathcal L_0(v^*)
&= \int_0^1 \mathbb E\!\left[\mathrm{Var}(Y \mid W_t)\right] dt \notag\\
&= \int_0^1 \mathbb E_{W_t}\!\left[\mathbb E_{U}\!\left[\mathrm{Var}(Y\mid W_t,U)\mid W_t\right]\right] dt
\;+\;
\int_0^1 \mathbb E_{W_t}\!\left[\mathrm{Var}\!\left(\mathbb E_{Y}[Y\mid W_t,U]\mid W_t\right)\right] dt \notag\\
&= \int_0^1 \mathbb E_{W_t,U}\!\left[\mathrm{Var}(Y\mid W_t,U)\right] dt
\;+\;
\int_0^1 \mathbb E_{W_t}\!\left[\mathrm{Var}\!\left(\mathbb E_{Y}[Y\mid W_t,U]\mid W_t\right)\right] dt \notag\\
&= \mathcal L_0(v_u^*)
+
\int_0^1 \mathbb E_{W_t}\!\left[\mathrm{Var}\!\left(\mathbb E_Y[Y\mid W_t,U]\mid W_t\right)\right] dt,
\end{align}
where the third line uses the tower property
$\mathbb E[\mathbb E[\cdot\mid W_t]]=\mathbb E[\cdot]$.
Rearranging proves the claimed decomposition.
\end{proof}

\begin{theorem}[Optimal Conditional Flow Matching Objective Decomposition under Inferred Latent]
\label{thm:formal_risk_decomp}
Let $Y := X_1 - X_0$ and $W_t := (X_t, X_{\mathrm{ctx}})$.
Suppose Assumption~\ref{ass:oracle_suff} (Oracle Sufficiency) and
Assumption~\ref{ass:lipschitz_latent} (Smooth Latent Space) hold.
Let $v_z^*$ and $v_u^*$ denote the population minimizers of the conditional
flow matching objective conditioned on the oracle latent variable $Z$
and the inferred representation $U$, respectively, i.e.,
\[
v_z^*(X_t,X_{\mathrm{ctx}},Z,t)=\mathbb E[Y\mid W_t,Z],
\qquad
v_u^*(X_t,X_{\mathrm{ctx}},U,t)=\mathbb E[Y\mid W_t,U].
\]
Then the corresponding optimal flow matching objectives satisfy
\begin{align}
\mathcal{L}_0(v_u^*)
\;\le\;
\mathcal{L}_0(v_z^*)
+
L^2 \int_0^1 \mathbb{E}\!\left[\mathrm{Var}\!\left(Z \mid W_t, U\right)\right] \, dt .
\end{align}
\end{theorem}

\begin{proof}
Fix any $t \in [0,1]$, by Assumption~\ref{ass:oracle_suff}, conditioning on $(W_t,Z)$ renders $U$
uninformative for predicting $Y$, and hence
\[
\mathbb E[Y \mid W_t,U,Z] = \mathbb E[Y \mid W_t,Z] = M_t.
\]

Applying the law of total variance to $Y$ conditional on $(W_t,U)$ yields
\begin{align}
\mathrm{Var}(Y\mid W_t,U)
&=
\mathbb{E}_{Z}\!\left[\mathrm{Var}(Y\mid W_t,U,Z)\mid W_t, U\right]
+
\mathrm{Var}\!\left(\mathbb{E}_Y[Y\mid W_t,U,Z]\mid W_t, U\right).
\end{align}
Substituting the oracle sufficiency condition (i.e., Assumption~\ref{ass:oracle_suff}) gives
\begin{align}
\mathrm{Var}(Y\mid W_t,U)
&=
\mathbb{E}_{Z}\!\left[\mathrm{Var}(Y\mid W_t,Z)\mid W_t, U\right]
+
\mathrm{Var}\!\left(\mathbb{E}_Y[Y\mid W_t,Z]\mid W_t, U\right).
\end{align} 
Taking expectation with respect to $(W_t,U)$ and integrating over $t$,
we obtain
\begin{align}
    \mathcal{L}_0
(v_u^*)&=\int_0^1 \mathbb{E}_{W_t,U}\left[\mathrm{Var}(Y\mid W_t,U)\right]dt\\
&=\int_0^1 \Big(\mathbb{E}_{W_t,U}\left[\mathbb{E}_Z\left[\mathrm{Var}(Y\mid W_t,Z)\mid W_t, U\right]\right]
+\mathbb{E}_{W_t,U}\left[\mathrm{Var}\left(\mathbb{E}_Y\left[Y\mid W_t,Z\right]\mid W_t, U\right)\right] \Big) dt \\
&=\int_0^1 \Big(\mathbb{E}_{W_t,U,Z}\left[\mathrm{Var}(Y\mid W_t,Z)\right]
+\mathbb{E}_{W_t,U}\left[\mathrm{Var}\left(\mathbb{E}_Y\left[Y\mid W_t,Z\right]\mid W_t, U\right)\right] \Big) dt\\
&=\mathcal{L}_0(v_z^*)+\int_0^1 \mathbb{E}_{W_t,U}\left[\mathrm{Var}\left(\mathbb{E}_Y\left[Y\mid W_t,Z\right]\mid W_t, U\right)\right] dt \label{eq_done_first_term}
\end{align}

We now bound the second term. Define $M_t=\mathbb{E}_Y\left[Y\mid W_t,Z\right]$, then by definition of conditional variance and tower property,
\begin{align}
&\mathbb{E}_{W_t,U}\left[\mathrm{Var}\left(\mathbb{E}_Y\left[Y\mid W_t,Z\right]\mid W_t, U\right)\right]\\
&=\mathbb{E}_{W_t,U}\left[ \mathbb{E}_{Z}\left[\|M_t-\mathbb{E}_{Z}\left[M_t\mid W_t, U\right]\|^2\mid W_t,U\right]
\right]\\
&=\mathbb{E}_{W_t,U,Z}\left[\|M_t-\mathbb{E}_{Z}\left[M_t\mid W_t, U\right]\|^2
\right]\\
&=\mathbb{E}_{W_t,U,Z}\left[\|\mathbb{E}_Y\left[Y\mid W_t,Z\right]-\mathbb{E}_{Z}\left[\mathbb{E}_Y\left[Y\mid W_t,Z\right]\mid W_t, U\right]\|^2
\right]
\end{align}

The random variable $\mathbb{E}_{Z}[M_t\mid W_t,U]=\mathbb E_Z\left[\mathbb{E}_Y\left[Y\mid W_t,Z\right] \mid W_t,U\right]$ is $\sigma(W_t,U)$-measurable
and is the $L^2$-orthogonal projection of $M_t$ onto the closed subspace
of $\sigma(W_t,U)$-measurable functions.
Therefore, for any measurable function
$g:\mathcal{W}\times\mathcal{U}\to\mathbb{R}^d$,
\begin{align}
\mathbb{E}_{W_t,U,Z}\!\left[
\left\|M_t - \mathbb{E}_{Z}[M_t\mid W_t,U]\right\|^2
\right]
&\le
\mathbb{E}_{W_t,U,Z}\!\left[
\left\|M_t - g(W_t,U)\right\|^2
\right]\\
&=\mathbb{E}_{W_t,U,Z}\!\left[
\left\| \mathbb{E}_Y\left[Y\mid W_t,Z\right] - g(W_t,U)\right\|^2
\right].
\label{eq:projection_ineq}
\end{align}

Let $\hat Z := \mathbb{E}_Z[Z\mid W_t,U]$, which is $\sigma(W_t,U)$-measurable.
Then $g(W_t,U):=m_t(W_t,\hat Z)$ is also $\sigma(W_t,U)$-measurable.
Applying Assumption~\ref{ass:lipschitz_latent} yields
\begin{align}
\mathbb{E}_{W_t,U}\!\left[
\mathrm{Var}\!\left(M_t\mid W_t,U\right)
\right]
&\le
\mathbb{E}_{W_t,U,Z}\!\left[
\|m_t(W_t,Z)-m_t(W_t,\hat Z)\|^2
\right] \\
&\le
L^2\,\mathbb{E}_{W_t,U,Z}\!\left[\|Z-\hat Z\|^2\right] \\
&=
L^2\,\mathbb{E}\!\left[\mathrm{Var}(Z\mid W_t,U)\right].
\end{align}
Substituting into~\eqref{eq_done_first_term} completes the proof.
\end{proof}

\section{More Experiment results \label{more_exp_results}}

\begin{figure}[H]
  \begin{center}
\centerline{\includegraphics[width=\linewidth]{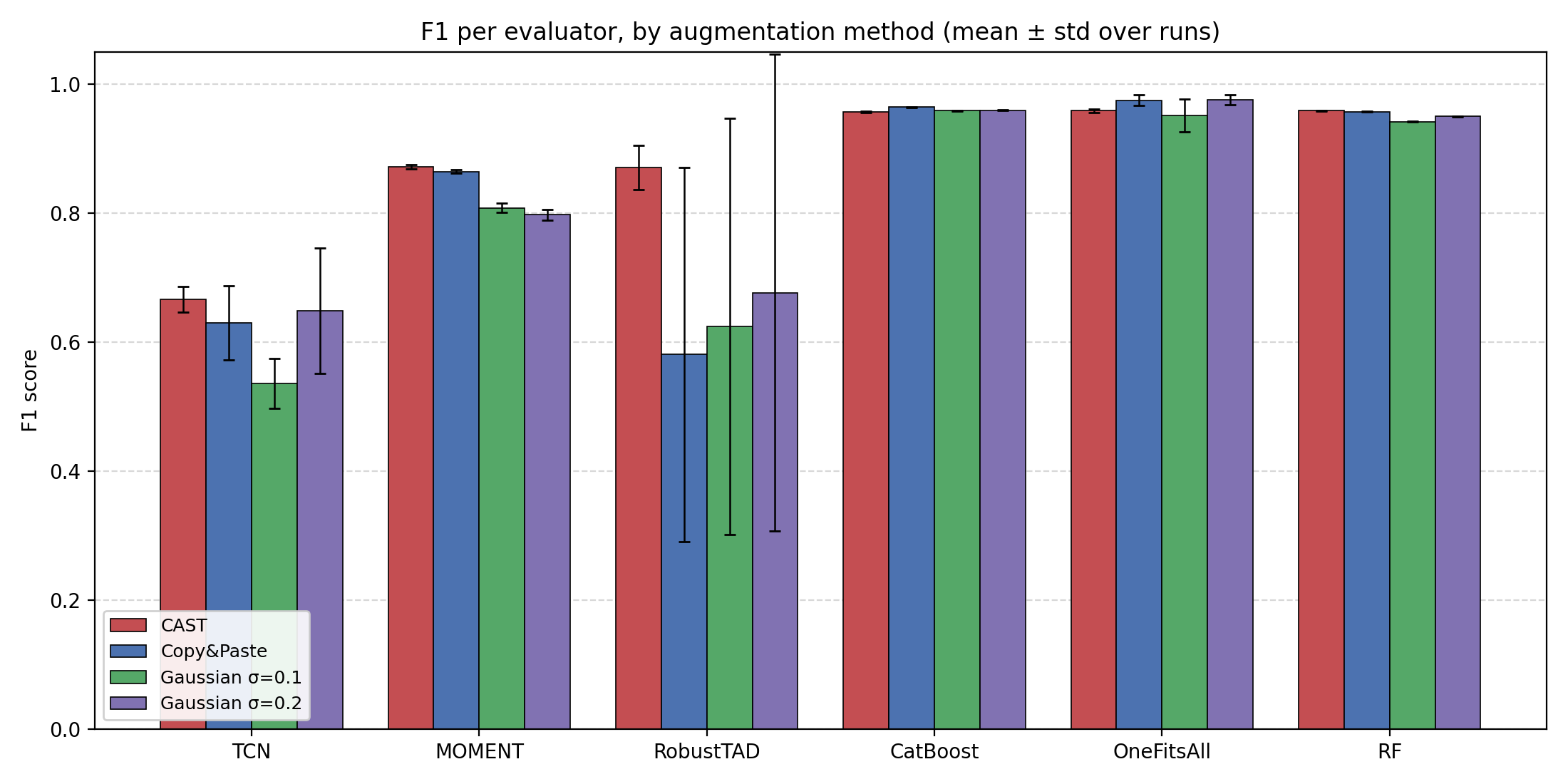}}
    \caption{Comparison between \ourmethod\ and simple baselines. Mean and standard deviation of F1 scores are reported over 5 runs.}  \label{fig:compared_with_simple_baselines}
  \end{center}
\end{figure}

\paragraph{Comparison with Simple Baselines}
We further provide a detailed comparison between \ourmethod\ and simple baselines including Copy-and-Paste and Gaussian Noise Injection in Figure~\ref{fig:compared_with_simple_baselines}. We observe that the performance of different methods varies across evaluators. For CatBoost, OneFitsAll, and Random Forest, all augmentation methods achieve similarly strong performance, indicating that these evaluators are relatively insensitive to the choice of augmentation once a certain performance level is reached. In contrast, for TCN, Moment, and RobustTAD, CAST consistently achieves higher F1 scores compared to baseline augmentation methods, with a clear margin. This suggests that CAST provides more informative and structured synthetic anomalies, which are particularly beneficial for more challenging time series anomaly detection settings.


\begin{table}[t]
\centering
\caption{Comparison of FID on Incart dataset. Higher is better.}
\label{incart_result}
\begin{tabular}{lccc}
\toprule
Method & Random Forest & OneFitsAll & Avg \\
\midrule
CAST        & 0.479 & 0.917 & 0.698 \\
\midrule
FlowTS      & 0.393 & 0.776 & 0.5845 \\
Diffusion-TS& 0.395 & 0.778 & 0.5865 \\
TimeVAE     & 0.382 & 0.407 & 0.3945 \\
C-GATS      & 0.369 & 0.221 & 0.295 \\
GenIAS      & 0.384 & 0.162 & 0.273 \\
\bottomrule
\end{tabular}
\end{table}

\paragraph{Results on Noisier Dataset.}
To further evaluate the performance of different methods on high-dimensional data, we conduct additional experiments on the INCART dataset, a 12-lead ECG dataset.
The results are shown in Table~\ref{incart_result}. We observe that CAST consistently outperforms all baselines across both anomaly detectors, achieving the highest average F1 score. In particular, CAST shows a clear advantage over generative baselines such as FlowTS, Diffusion-TS, and TimeVAE, as well as anomaly-focused methods including C-GATS and GenIAS.

\begin{figure}[bht]
  \begin{center}
\centerline{\includegraphics[width=\linewidth]{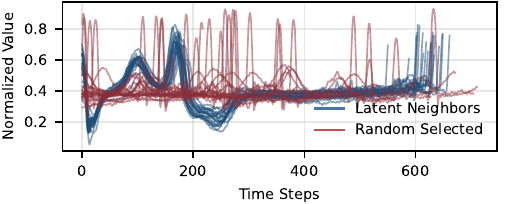}}
    \caption{
Latent structure diagnostic.
Time series neighbors in the trained VQVAE's latent space are shown in blue. Randomly selected time series are shown in red.}    \label{fig:latent_structure_diagnostic}
  \end{center}
\end{figure}

\paragraph{Qualitative Anlysis on VQVAE: } Figure~\ref{fig:latent_structure_diagnostic} examines whether the discrete latent space learned by the VQ-VAE encodes meaningful temporal structure. Time-series segments corresponding to latent neighbors exhibit strong morphological similarity in the time domain. By contrast, randomly selected time series shows substantially larger variability and inconsistent temporal patterns. This observation suggests that the VQ-VAE latent space organizes anomaly-related structures in a semantically meaningful manner. Moreover, the residual variation among latent neighbors indicates that latent proximity alone is insufficient to fully characterize anomaly morphology, motivating the fine-grained anomaly structure conditioning in \ourmethod.




\begin{table}[H]
\centering
\caption{Comparison between modified GenIAS and original GenIAS across multiple anomaly detectors (F1 score). The average performance is reported in the last column.}
\label{tab:genias_compare}
\resizebox{\linewidth}{!}{
\begin{tabular}{c|cccccc|c}
\toprule
Method & Random Forest & CatBoost & RobustTAD & TCN & Moment & OneFitsAll & Avg \\
\midrule

\multicolumn{8}{c}{\textbf{MIT-DB}} \\
\midrule
Modified & 0.609 & 0.410 & 0.782 & 0.604 & 0.590 & 0.837 & \textbf{0.6387} \\
Original & 0.480 & 0.425 & 0.733 & 0.424 & 0.531 & 0.754 & 0.5578 \\

\midrule
\multicolumn{8}{c}{\textbf{SVDB}} \\
\midrule
Modified & 0.632 & 0.558 & 0.581 & 0.733 & 0.495 & 0.903 & \textbf{0.6503} \\
Original & 0.620 & 0.434 & 0.063 & 0.657 & 0.268 & 0.832 & 0.4790 \\

\midrule
\multicolumn{8}{c}{\textbf{QTDB}} \\
\midrule
Modified & 0.443 & 0.427 & 0.165 & 0.527 & 0.392 & 0.813 & \textbf{0.4612} \\
Original & 0.475 & 0.374 & 0.000 & 0.205 & 0.290 & 0.920 & 0.3773 \\

\midrule
\multicolumn{8}{c}{\textbf{PV}} \\
\midrule
Modified & 0.000 & 0.0015 & 0.078 & 0.076 & 0.325 & 0.114 & \textbf{0.0991} \\
Original & 0.000 & 0.0000 & 0.000 & 0.075 & 0.223 & 0.170 & 0.0780 \\

\midrule
\multicolumn{8}{c}{\textbf{Metro}} \\
\midrule
Modified & 0.000 & 0.0000 & 0.000 & 0.000 & 0.227 & 0.314 & 0.0902 \\
Original & 0.000 & 0.0000 & 0.000 & 0.000 & 0.204 & 0.459 & \textbf{0.1105} \\

\bottomrule
\end{tabular}
}
\end{table}

\paragraph{Additional analysis on GenIAS.}
To better understand the role of anomaly-specific modeling, we conduct an additional experiment based on GenIAS. 
In particular, we implement a modified variant of GenIAS that incorporates explicit anomaly-aware modeling during training, aiming to encourage the model to capture structured anomaly patterns rather than relying solely on stochastic latent perturbations. The results are summarized in Table~\ref{tab:genias_compare}. We observe that the modified version achieves improved performance on most datasets, with consistent gains in average F1 score across multiple anomaly detectors. This trend suggests that incorporating explicit anomaly modeling can be beneficial for generating more realistic and useful anomalous samples. We note that this experiment is intended as a controlled analysis to better understand the effect of anomaly structure modeling, rather than a direct comparison between methods.

\section{Limitations \label{app:limitation}}
Although CAST demonstrates strong performance across multiple real-world datasets, several limitations remain. In practical applications, anomaly generation systems may impose stringent requirements on inference efficiency. Although CAST adopts Flow Matching for efficient generation, it still relies on multi-step inference. Exploring more advanced single-step generative paradigms could further improve its applicability in real-world deployment scenarios. In addition, the theoretical analysis presented in this work is primarily intended to provide intuition for the design of CAST rather than establish a comprehensive theoretical foundation. Developing a deeper theoretical understanding of anomaly generation and conditional generative modeling for time series remains an important direction for future research.

\section{Impact Statement \label{app:impact}}

This paper presents work whose primary goal is to advance the field of machine learning, specifically in the area of anomalous time series generation. While the proposed method may have potential applications in safety-critical domains such as system monitoring and fault diagnosis, we do not foresee any immediate negative societal or ethical impacts arising from this work. We believe the techniques developed in this paper contribute to improving the robustness and reliability of data-driven systems, and no additional ethical concerns beyond those commonly associated with machine learning research are identified.

\end{document}